\documentclass[11pt]{article}
\usepackage[ngerman,english]{babel}
\usepackage{amsmath, amsxtra, amsfonts, amssymb, amstext}
\usepackage{amsthm}
\usepackage{booktabs}
\usepackage{nicefrac}
\usepackage{xspace}
\usepackage[noadjust]{cite}
\usepackage{url}
\usepackage{graphics}
\usepackage[usenames,dvipsnames]{xcolor}
\usepackage[colorlinks]{hyperref}
\definecolor{linkblue}{rgb}{0.1,0.1,0.8}
\hypersetup{colorlinks=true,linkcolor=linkblue,filecolor=linkblue,urlcolor=linkblue,citecolor=linkblue}
\usepackage[algo2e,ruled,vlined,linesnumbered]{algorithm2e}
\usepackage{wrapfig}
\usepackage{pdfpages}
\usepackage{fullpage}

\newtheorem{theorem}{Theorem}
\newtheorem{lemma}[theorem]{Lemma}

\newtheorem{definition}[theorem]{Definition}

\newcommand{\R}{\mathbb{R}}
\newcommand{\Z}{\mathbb{Z}}

\renewcommand{\epsilon}{\varepsilon}

\DeclareMathOperator{\E}{E}

\DeclareMathOperator{\mut}{mut}
\DeclareMathOperator{\cross}{cross}

\newcommand{\Bin}{\mathcal{B}}

\newcommand{\assign}{\leftarrow}

\newcommand{\oea}{$(1 + 1)$~EA\xspace}

\newcommand{\ga}{$(1 + (\lambda,\lambda))$~GA\xspace}

\newcommand{\onemax}{\textsc{OneMax}\xspace}
\newcommand{\OM}{\textsc{Om}\xspace}

\newcommand{\leadingones}{\textsc{LeadingOnes}\xspace}

\newcommand{\const}{{C_{0}}}
\newcommand{\barlam}{\bar{\lambda}}

\begin{document}
    
\title{Optimal Parameter Choices Through Self-Adjustment: Applying the 1/5-th Rule in Discrete Settings}

\author{Benjamin Doerr$^{1}$,  
Carola Doerr$^{2,3}$}
 
\date{
$^1$\'Ecole Polytechnique, Paris--Saclay, Palaiseau, France\\
$^2$Sorbonne Universit\'es, UPMC Univ Paris 06,  UMR 7606, LIP6, Paris, France\\
$^3$CNRS, UMR 7606, LIP6, Paris, France\\[2ex]
\today
}

\sloppypar
\maketitle
 
\begin{abstract}
While evolutionary algorithms are known to be very successful for a broad range of applications, the algorithm designer is often left with many algorithmic choices, for example, the size of the population, the mutation rates, and the crossover rates of the algorithm. These parameters are known to have a crucial influence on the optimization time, and thus need to be chosen carefully, a task that often requires substantial efforts. Moreover, the optimal parameters can change during the optimization process. It is therefore of great interest to design mechanisms that dynamically choose best-possible parameters. An example for such an update mechanism is the one-fifth success rule for step-size adaption in evolutionary strategies. While in continuous domains this principle is well understood also from a mathematical point of view, no comparable theory is available for problems in discrete domains.

In this work we show that the one-fifth success rule can be effective also in discrete settings. We regard the $(1+(\lambda,\lambda))$~GA proposed in [Doerr/Doerr/Ebel: From black-box complexity to designing new genetic algorithms, TCS 2015]. We prove that if its population size is chosen according to the one-fifth success rule then the expected optimization time on \textsc{OneMax} is linear. This is better than what \emph{any} static population size $\lambda$ can achieve and is asymptotically optimal also among all adaptive parameter choices.  
\end{abstract}
%

\sloppy{
\section{Introduction}

It is widely acknowledged that setting the parameters of evolutionary algorithms (EA) is one of the key difficulties in evolutionary optimization. Eiben, Hinterding, and Michalewicz~\cite{EibenHM99} call this challenge ``one of the most important and promising areas of research in evolutionary computation''. This statement retains its topicality 15 years after the original publication of~\cite{EibenHM99} as many talks at evolutionary computation conferences certify. We also understand today that even small changes in the parameters can yield to exponential performance gaps of the regarded algorithms~\cite{DrosteJW02,DoerrJSWZ13}.

Substantial research efforts have been undertaken to find good parameter settings for general EAs, see for example~\cite{DeJong75}. Around the same time it has been discovered that it may be sub-optimal to use a fixed set of parameters throughout the whole optimization process. It was suggested instead to change the parameters of the algorithms by some dynamic update rules, often using some sort of feedback of the fitness landscape that the algorithm is facing. For example, it can be beneficial in earlier parts of the process to invest in \emph{exploration} of the fitness landscape, while the algorithm should become more stable and focus on one or few areas of attraction in the later \emph{exploitation} phase(s). 

Interestingly, while in continuous domains parameter control\footnote{We use here and in the following a slightly adapted version of the terminology for parameter setting suggested in~\cite{EibenHM99}. This terminology is summarized in Section~\ref{sec:classification} and Figure~\ref{fig:classification}.} is analyzed also theoretically~\cite{AugerH13, Jagerskupper05, HansenGO95}, adaptive parameter choices play only a marginal role in theoretical investigations of EAs for discrete search spaces, 
the few exceptions showing an advantage of adaptive parameter settings being the fitness-dependent mutation rates for the (1+1) evolutionary algorithm on \leadingones analyzed in~\cite{BottcherDN10} and the fitness-dependent choice of the population size in~\cite{DoerrDE15} for the \ga on \onemax, respectively. Both these adaption schemes, however, require a very solid understanding of the problem at hand and are thus likely to be of limited practical relevance. Another example from the discrete EA literature are reductions of the parallel runtime (but not the total optimization time) for several test functions when doubling the number of parallel instances in a parallel EA after each unsuccessful iteration~\cite{LassigS11}. 

With this work we provide a first example for a discrete optimization problem where a self-adjusting (i.e., adaptive, but not fitness-dependent) parameter choice yields an expected optimization time that is better by more than a constant factor than any static parameter choice. The parameter update rule is extremely simple and does not require any problem-specific insights. More precisely, we analyze the runtime of the already mentioned \ga with population sizes chosen according to the one-fifth success rule on the generalized \onemax problem. We also show that the one-fifth success update scheme is optimal in this setting, that is, no alternative update mechanism can yield a significantly smaller runtime. In fact, we show that, throughout the whole optimization process, the one-fifth success rule suggests parameter settings that closely follow the theoretically best possible choices.

\subsection{The One-Fifth Rule}
One of the earliest adaptive update rules suggested in the evolutionary computation literature is the \emph{one-fifth success rule}. It was independently discovered in~\cite{Rechenberg, Devroye72, SchumerS68} and constitutes today one of the best known and most widely applied techniques in parameter control. Several empirical results (cf.~\cite{EibenS03} and references therein) suggest that EAs using the one-fifth rule for adaptive parameter control are quite capable of finding optimal or close to optimal parameter settings. Since the parameters are updated without the intervention of the user, such update mechanisms are a very convenient way to minimize parameter tuning efforts. Furthermore, the one-fifth success rule does not require any problem-specific knowledge and is thus widely applicable.

Originally, the one-fifth rule was designed to control the step size of evolution strategies. In intuitive terms, it suggests that if the probability to create an offspring of better than current-best fitness is greater than $1/5$, then the step size should be increased, while it should be decreased if the probability is lower than $1/5$. Today, this rule has found applications much beyond the adaptation of the step size. Here in this work we use it for adjusting the offspring population size of a genetic algorithm.

Without going into details, we note that many other adaptive update rules have been experimented with in the evolutionary computation literature, cf.\! \cite{EibenHM99}, \cite[Chapter 8]{EibenS03}, and~\cite{KarafotiasHE15} for excellent surveys. 

\subsection{The \texorpdfstring{\ga}{(1+(lambda,lambda))~GA}}

We regard the \ga, which has been proposed in~\cite{DoerrDE15} as a first example of an evolutionary algorithm optimizing the generalized \onemax problem using $o(n \log n)$ function evaluations. For the theory of evolutionary algorithms, this is a big success as it shows for the first time that even for such simple problems the usage of crossover can be beneficial (all previously known evolutionary algorithms need $\Omega(n \log n)$ function evaluations in expectation to optimize the generalized \onemax problem).

An important parameter of the \ga is $\lambda$, the number of offspring generated in the mutation phase and the subsequent crossover phase of the algorithm. In the original paper~\cite{DoerrDE15} it is shown that for $\lambda=O(\sqrt{\log n})$ the expected optimization time of the \ga on \onemax is $O(n \sqrt{\log n})$. In~\cite{DoerrD15exact} we improve the runtime bound to $\Theta(\max\{n \log(n) / \lambda, n \lambda \log\log(\lambda)/\log(\lambda)\})$. This expression is minimized for $\lambda = \Theta(\sqrt{\log(n) \log\log(n)/\log\log\log(n)})$, giving an optimization time of $\Theta(n \sqrt{\log(n) \log\log\log(n) / \log\log(n)})$. While this exact expression is irrelevant for the purposes of the present paper, it is important to note that this tight bound is super-linear for every possible choice of $\lambda$.

It was also observed in~\cite{DoerrDE15} that the \ga has only linear expected optimization time on \onemax when the population size $\lambda$ is chosen adaptively depending on the fitness of the current-best individual. While theoretically appealing, this result has limited practical implications since the proposed fitness-dependent parameter choice crucially requires a very good understanding of the optimization process and thus, of the problem at hand. This is testified by the optimal relation of the population size to the current fitness-distance $d$ to the optimum, which is $\Theta(\sqrt{n/d})$. Guessing such a functional relationship for real-world optimization problems is typically not doable with reasonable efforts.  

Interestingly, an alternative approach based on the one-fifth success rule was suggested in~\cite{DoerrDE15}. In a series of experimental evaluations it was shown that if the population size $\lambda$ is chosen according to this rule, the performance of the resulting \ga is among the best ones for a series of test problems. In that algorithm the population size is increased if no improvement has happened in the last iteration, while it is decreased otherwise. More precisely, for a suitable constant $F > 1$, the population size parameter $\lambda$ is multiplied by $F^{1/4}$ after each iteration in which the fitness of the current-best individual could not be improved, and it is divided by $F$ otherwise, i.e., if the fitness of the current-best search point increased in that iteration. 

\subsection{Our Results}

While the observations made in~\cite{DoerrDE15} are purely empirical, we provide with this work a theoretical analysis of the suggested self-adjusting \ga. We prove that the suggested implementation of the one-fifth success rule yields a linear expected optimization time of the \ga on the generalized \onemax problem. As noted above this is better than what any static parameter choice can achieve and is also best possible among all comparison-based algorithms as we shall comment in Section~\ref{sec:lower}. In particular, our bound shows that the one-fifth success rule suggests population sizes that are asymptotically optimal among all possible (static and dynamic) choices.

To the best of our knowledge, this is the first time that in a discrete search environment a self-adjusting parameter choice is shown to be superior to any static choice. Indeed the \ga is the first proven example in discrete evolutionary algorithmics where a non-fitness-dependent parameter choice reduces the optimization time by more than a constant factor. The results that come closest to this are the mentioned results from~\cite{BottcherDN10,LassigS11}, which are either constant factor reductions of the expected runtime (in case of~\cite{BottcherDN10}) or  reductions of the parallel expected runtime but not the total number of function evaluations (in case of~\cite{LassigS11}).

Our proof gives some general insights in the working principles of adaptive parameter choices in discrete domains, which hopefully lead to future applications of this approach in discrete search. 

Our paper is organized as follows. We first introduce the \ga with static population sizes, give background on the generalized \onemax problem, and recall known bounds for the expected optimization time of the \ga on \onemax functions in Section~\ref{sec:algorithm}. In Section~\ref{sec:selfadaptive} we present the \ga with self-adjusting parameter choices along with a brief summary of the mentioned (slightly adapted) classification scheme of Eiben, Hinterding, and Michalewicz~\cite{EibenHM99} for parameter settings. We also present in Section~\ref{sec:selfadaptive} the runtime analysis of the self-adjusting \ga (see Section~\ref{sec:runtime}), followed by a discussion of the general insights obtained through that analysis (Section~\ref{sec:insights}). Finally, we show in Section~\ref{sec:lower} that the one-fifth success rule suggests optimal or close to optimal parameter settings. 

\section{The \texorpdfstring{$(1+(\lambda,\lambda))$}{(1+(lambda,lambda))}~GA with Static Population Size}
\label{sec:algorithm} 

Adopting the conventions and notation from~\cite{DoerrDE15}, we regard here in this work only search spaces with bit string representations, we write $x=x_1 \ldots x_n$, $[k]:=\{1,2, \ldots, k\}$, and $[0..k]:=[k]\cup\{0\}$ for any bit string $x \in \{0,1\}^n$ and any non-negative integer $k$.
By $\Bin(n,p)$ we denote the binomial distribution with $n$ trials and success probability $p$; i.e., $\Bin(n,p)(\ell)=\binom{n}{\ell} p^{\ell}(1-p)^{n-\ell}$ for any $\ell \in [0..n]$.

The \ga is given in Algorithm~\ref{alg:GA}. It uses the following two variation operators. 
\begin{itemize}
	\item The unary mutation operator $\mut_{\ell}(\cdot)$ which, given some $x \in \{0,1\}^n$, creates from $x$ a new bit string $y$ by flipping exactly $\ell$ bit entries in it.
	\item The binary crossover operator $\cross_c(\cdot,\cdot)$ with crossover probability $c$, which, given two bit strings $x$ and $x'$, chooses $y:=\cross_c(x,x')$ by choosing for each $i \in [n]$ $y_i:=x'_i$ with probability $c$ and choosing $y_i=x_i$ otherwise. 
\end{itemize}
Thus, after a random initialization of the algorithm, in each iteration the following steps are performed. 
\begin{itemize}
	\item In the \emph{mutation phase}, $\lambda$ offspring are sampled from the current-best solution $x$ by applying $\lambda$ times independently the mutation operator $\mut_\ell(\cdot)$ to $x$, where the step size $\ell$ is chosen at random from $\Bin(n,p)$ before the generation of the first offspring.
	\item In the \emph{crossover phase}, $\lambda$ offspring are created from~$x$ and (one of) the best of the $\lambda$~offspring from the mutation phase, $x'$, by sampling independently from $\cross_c(x,x')$.
	\item \emph{Elitist selection}: The best of the individuals of the crossover phase replaces $x$ if its fitness is at least as large as the fitness of $x$. If there are several offspring with best fitness, we disregard those that are equal to $x$ and choose one of the remaining ones uniformly at random.
\end{itemize}
Throughout this paper we shall use $p=\lambda/n$ and $c=1/\lambda$, choices which are well justified in~\cite[Sections~2 and~3]{DoerrDE15}. Only in Section~\ref{sec:lower} we regard arbitrary choices of the parameters $\lambda$, $p$, and $c$. 

\begin{algorithm2e}%
	\textbf{Initialization:} 
	Sample $x \in \{0,1\}^n$ uniformly at random and query $f(x)$\;
 \textbf{Optimization:}
\For{$t=1,2,3,\ldots$}{
\underline{\textbf{Mutation phase:}}\\
\Indp
Sample $\ell$ from $\Bin(n,p)$\label{line:L}\;
\For{$i=1, \ldots, \lambda$\label{line:mutstart}}{
Sample $x^{(i)} \assign \mut_{\ell}(x)$ and query $f(x^{(i)})$\;
}
Choose $x' \in \{x^{(1)}, \ldots, x^{(\lambda)}\}$ with $f(x')=\max\{f(x^{(1)}), \ldots, f(x^{(\lambda)})\}$ u.a.r.\label{line:mutend}\;
\Indm
\underline{\textbf{Crossover phase:}}\\
\Indp
\For{$i=1, \ldots, \lambda$\label{line:costart}}{
Sample $y^{(i)} \assign \cross_{c}(x,x')$ and query $f(y^{(i)})$\label{line:co}\; 
}
If exists, choose $y \in \{y^{(1)}, \ldots, y^{(\lambda)}\} \setminus \{x\}$ 
with 
$f(y)=\max\{f(y^{(1)}), \ldots, f(y^{(\lambda)})\}$ u.a.r.; otherwise, set $y:=x$\label{line:coend}\;
\Indm
\underline{\textbf{Selection step:}}
\lIf{$f(y)\geq f(x)$}{$x \assign y$\;
}
}
\caption{The \ga with offspring population size $\lambda$, mutation probability $p$, and crossover probability $c$.}
\label{alg:GA}
\end{algorithm2e}


As performance measure we regard the expected \emph{running time} of the \ga, that is, the expected number of function evaluations that the algorithm performs until it evaluates for the first time an optimal search point $x \in \arg\max f$. This is the common measure in runtime analysis and is sometimes referred to as the expected \emph{optimization time}. Note that for algorithms performing more than one fitness evaluation per iteration, such as the \ga, the expected runtime can be much different from the expected number of \emph{iterations} (generations).

\subsection{The Generalized OneMax Problem}
\label{sec:onemax}

In~\cite[Section 4]{DoerrDE15} the \ga is analyzed by experimental means. The results show that it performs well on \onemax functions, linear functions with random weights, and royal road functions. A theoretical investigation, an improved bound of which is stated below in Theorem~\ref{thm:previous2}, however, is currently available only for the generalized \onemax problem. As this problem is also the focus of our present work, we give a short introduction here.

The classical \onemax function counts the number of ones in a bit string. Optimizing it therefore corresponds to finding the all-ones bit string. Of course, we want the performance of an evolutionary algorithm to be independent of the problem encoding. More specifically, the algorithms that we typically regard have exactly the same optimization behavior on any \emph{generalized \onemax function}
\begin{align*}
\OM_z : \{0,1\}^n \to \R; x \mapsto |\{i \in \{1, \ldots, n\} \mid x_i = z_i\}|,
\end{align*}
which counts the number of positions in which the bit strings $x$ agrees with the \emph{target string} $z$. It is easy to see that $z$ is the unique (global and local) optimum of $\OM_z$. Note also that the classic \onemax function counting the number of ones in a bit string is the function $\OM_z$ with $z=(1,\ldots,1)$. It is therefore justified to call the set $\{\OM_z \mid z \in \{0,1\}^n \}$ the \emph{generalized \onemax problem}. For convenience we often drop the word ``generalized'' in the following. All statements made below hold for arbitrary \onemax functions.

The expected runtime of most search heuristics on \onemax is $\Omega(n \log n)$, due to a phenomenon called coupon collector's problem (see, e.g., \cite[Section 1.5]{Doerr11bookchapter} or \cite{MotwaniR95}). In intuitive terms, the argument is as follows. When the initial bit string is taken uniformly at random from $\{0,1\}^n$ then each bit has a probability of $1/2$ of being in the wrong initial configuration, i.e., it has to be touched with probability $1/2$. The coupon collector's problem states that if we touch one random bit at a time, then it takes $\Omega(n \log n)$ iterations until we have touched each bit at least once. Since many evolutionary algorithms (including, for example, \oea and Randomized Local Search) change on average one bit per iteration, this implies the $\Omega(n \log n)$ bound. 
It was a long-standing open question whether genetic algorithms can perform better on \onemax than this lower bound. Sudholt~\cite{Sudholt12} gave a first example of a crossover-based genetic algorithm outperforming the \oea on \onemax. But while his (2+1)~GA (for a suitably chosen mutation rate) is better by a constant factor than the \oea, it does not improve upon RLS. The \ga thus gave a first positive answer to this question, as we recall in the next section.

\subsection{Runtimes for Static and Fitness-Dependent Population Sizes}
\label{sec:known}
The following statement, proven in~\cite{DoerrD15exact}, provides a tight bound for the expected runtime of the \ga on \onemax.
\begin{theorem}[\cite{DoerrD15exact}]
\label{thm:previous2}
  The expected optimization time of the \ga with 
  $p = \lambda / n$ and 
  $c = 1 / \lambda$ on every generalized \onemax function is 
  \begin{align*}
  \Theta\left(\max\left\{\frac{n \log(n)}{\lambda}, \frac{n \lambda \log\log(\lambda)}{\log(\lambda)}\right\}\right).
  \end{align*}
Consequently, 
$\lambda = \Theta(\sqrt{\log(n) \log\log(n) / \log\log\log(n)})$ is the optimal choice for the parameter $\lambda$ and this yields an expected optimization time of $\Theta(n \sqrt{\log(n) \log\log\log(n) / \log\log(n)})$.
\end{theorem}

It is also known~\cite{DoerrDE15} that a fitness-dependent (and thus, inherently non-static) choice of the population size can decrease the runtime even further. 

\begin{theorem}[Theorem~8 in~\cite{DoerrDE15}]
\label{thm:fitnessdependent}
The expected runtime of the \ga with $p=\lambda/n$, $c=1/\lambda$, and fitness-dependent choice of $\lambda:=\lceil \sqrt{n/(n-f(x))} \rceil$ on every generalized \onemax function is linear in~$n$.
\end{theorem}

We will use the latter bound in our analysis of the self-adjusting \ga. More precisely, we show that the population sizes suggested by the one-fifth success rule typically do not deviate much from the fitness-dependent choice analyzed in Theorem~\ref{thm:fitnessdependent}. This observation has also been made experimentally in~\cite{DoerrDE15}. Figure~\ref{fig:plot}, taken from~\cite{DoerrDE15} (Figure~5 in that paper) shows the close relationship between the self-adjusting population sizes (in red) and the optimal fitness-dependent ones (in black) for a typical run of the \ga on \onemax. 

\begin{figure}
\begin{center}
\includegraphics[width=.7\linewidth]{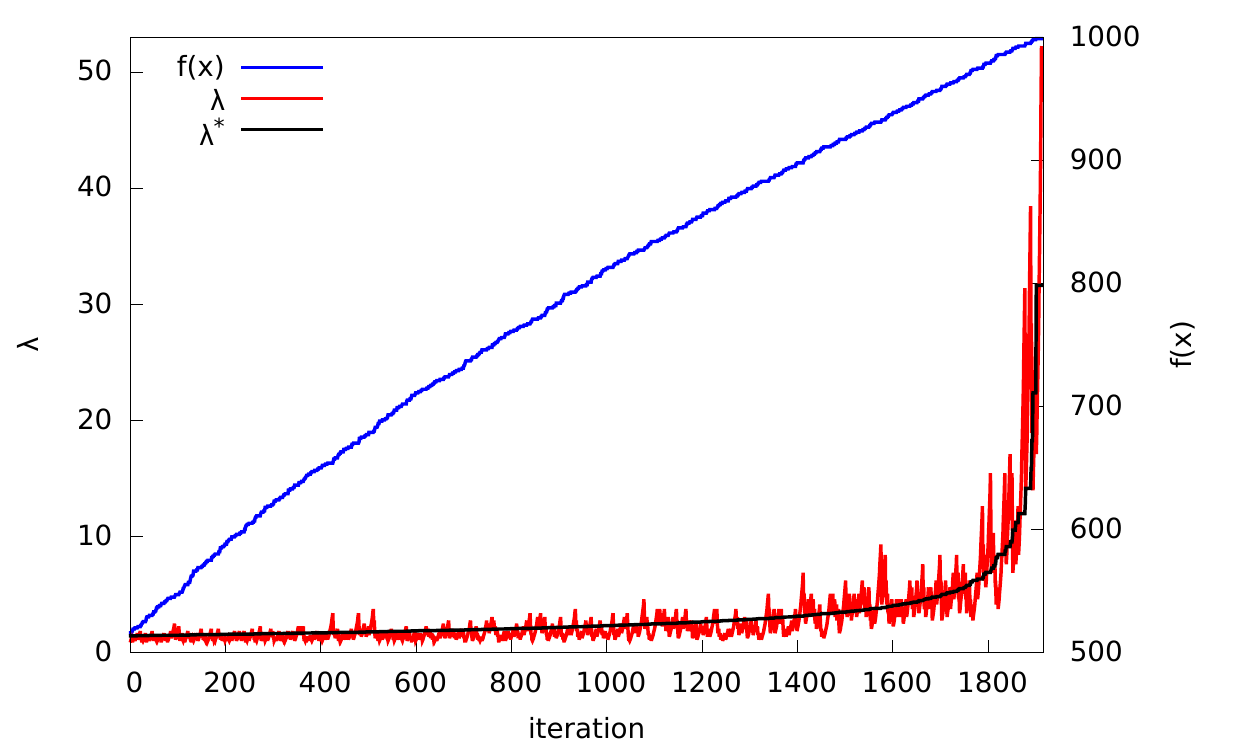}
\end{center}
\caption{{Evolution of $f(x)$ and $\lambda$ in one representative run of the \ga with self-adjusting~$\lambda$ and update strength $F=1.5$ on \onemax with $n=1000$. The bottom non-rugged curve plots the fitness-dependent choice $\lambda^*=\sqrt{n/(n-f(x))}$ analyzed in~\cite{DoerrDE15} (cf.\! Theorem~\ref{thm:fitnessdependent}).}}
\label{fig:plot}
\end{figure}

Two milestones in the analysis of the \ga in~\cite{DoerrDE15} are the success probabilities of the mutation and the selection phase, respectively. Since we shall make use of these two bounds, we briefly repeat them below. 

Note for Lemma~\ref{lem:phase1} that for an offspring $x'$ of $x$ with $\OM_z$-value greater than $\OM_z(x)-\ell$ there exists at least one position $i$ such that $x'_i=z_i$ while $x_i \neq z_i$. It is therefore possible to extract in the crossover phase this entry from $x'$ and thus increasing the overall fitness of the current best search point. This is why we call the  mutation phase successful if $\OM_z(x')>\OM_z(x)-\ell$ holds.

\begin{lemma}[Lemma~5 in~\cite{DoerrDE15}]
\label{lem:phase1}
 In the notation of Algorithm~\ref{alg:GA}, for all $\ell$ and $x$, the probability that in the mutation phase a search point $x'$ is created with $\OM_z(x')>\OM_z(x)-\ell$ is at least $1-\left(\frac{\OM_z(x)}{n}\right)^{\lambda \ell}$.
\end{lemma}

\begin{lemma}[Lemma~6 in~\cite{DoerrDE15}]
\label{lem:phase2}
In the notation of  Algorithm~\ref{alg:GA}, consider fixed outcomes of $\ell$, $x$, and $x'$. Then the random outcome $y$ of the crossover phase satisfies
\begin{align*}
Pr[ \OM_z(y)> \OM_z(x) \mid & \OM_z(x')>\OM_z(x)-\ell]\\
& \quad \geq 
1-\left(1-c (1-c)^{\ell-1}\right)^{\lambda}.
\end{align*}
\end{lemma}

\section{The \texorpdfstring{$(1+(\lambda,\lambda))$}{(1+(lambda,lambda))}~GA with Self-Adjusting Population Sizes}
\label{sec:selfadaptive}

Being the first provable super-constant speed-up via a fitness-dependent parameter choice, the linear optimization time obtained in~\cite{DoerrDE15} is a big success in the theory of evolutionary algorithms. From the practical point of view, though, the question remains how in an actual application the user of the \ga would guess the fitness-dependent optimal choice of $\lambda$. In this section, we show that this is not needed. A self-adjusting choice inspired by the classic one-fifth rule can give the same (optimal, as the result in Section~\ref{sec:lower} shows) linear optimization time. To the best of our knowledge, this is the first result proving a reduced optimization time via parameter self-adjustment in discrete search spaces. We are optimistic that our approach can be applied to other discrete problems. At the end of this section, we give some general hints that might be useful for such purposes.

\subsection{Terminology for Parameter Settings}
\label{sec:classification}

\begin{figure}[t]
\begin{center}
\includegraphics[width=0.5\linewidth]{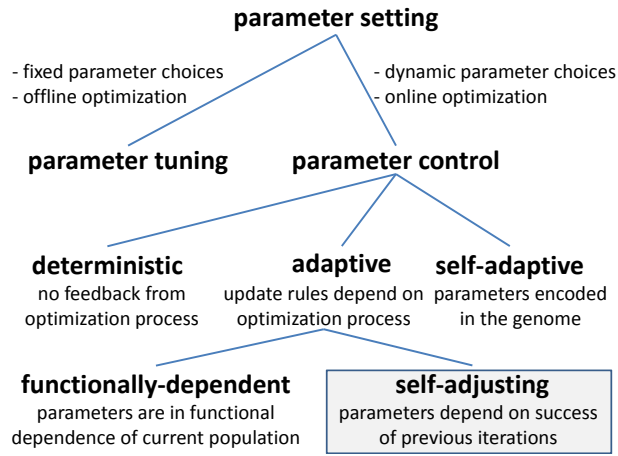}
\end{center}
\caption{An extended version of the classification scheme from~\cite{EibenHM99}. We regard in this work self-adjusting parameter choices.}
\label{fig:classification}
\end{figure}

Since the literature is unanimous with respect to the terminology for parameter settings, we have adopted and slightly extended in this work the taxonomy of Eiben, Hinterding, and Michalewicz~\cite{EibenHM99}. Figure~\ref{fig:classification}, an adapted version of Figure~1 in~\cite{EibenHM99}, illustrates this classification, which we briefly summarize below.

The efforts of choosing the right parameters in an evolutionary algorithm is called \emph{parameter setting}. The first difference is between \emph{static} and \emph{dynamic} parameter settings. In the former the parameters are set before the actual run of the algorithm and they are not changed during the optimization process. In typical applications, a \emph{parameter tuning} step precedes the application of the EA. In this phase, suitable parameter choices are sought through initial experimental investigations, either for all parameters simultaneously, or in an iterative process. 

Optimizing dynamic parameter choices is called \emph{parameter control} in~\cite{EibenHM99}. Three principals are discussed: deterministic, adaptive, and self-adaptive parameter control. A dynamic parameter choice is called \emph{deterministic} if it does not depend on the fitness landscape encountered by the algorithm. That is, there is no feedback between the fitness values and the dynamic parameters.\footnote{As noted in~\cite{EibenHM99} this does not exclude randomized update schemes. A possibly better wording would therefore be \emph{fixed} or \emph{feedback-free} update rules.
} \emph{Adaptive} parameter choices are those dynamic rules where the update rule depends on the optimization process. Within this class (and this is different from the classification scheme proposed in~\cite{EibenHM99}) we distinguish between \emph{functionally-dependent} parameter choices (where the parameters depend only on the current state of the algorithm, i.e., the current population) and \emph{self-adjusting} adaptive choices, where the parameters depend on the success of (all) previous iterations. It is easy to see that the one-fifth success rule considered in this paper classifies as a self-adjusting parameter setting, while the fitness-dependent parameter choice considered in Theorem~\ref{thm:fitnessdependent} is an example for a functionally-dependent parameter choice. Finally, \emph{self-adaptive} parameter choices are encoded themselves in the genome of the search points and are subject to variation operators. The hope is that the better parameter choices yield better offspring and thus survive the evolutionary process.

\subsection{The Algorithm}
\label{sec:algo}

Recall that the one-fifth success rule in evolution strategies is used to change the step-size in a self-adjusting manner. When the empirical success probability is large, the step-size is increased to hopefully speed-up the exploration. When is it low, it is reduced to hopefully increase the chance of a success. This is done in a way that an average success probability of one fifth leads to no change of the step-size on average.

In discrete search spaces, naturally, things are very different. However, we can still come up with a natural variant of the one-fifth success rule. Note that in our \ga (with the suggested choices $p = \lambda/n$ and $c=1/\lambda$ for the mutation rate and the crossover rate, respectively), increasing $\lambda$ will increase the success probability of one iteration, however, at the price of an increased number of function evaluations, that is, higher runtime. Consequently, it makes sense to increase $\lambda$ when the empirical success probability is low (to speed up the process of finding an improvement), but to reduce it when the success probability is large (to hopefully save computational effort). 

Taking Auger's~\cite{Auger09} implementation of the one-fifth success rule as example, we design the following self-adjusting version of the \ga, see also Algorithm~\ref{alg:GAself}. After an iteration that led to an increase of the fitness of $x$ (``success''), indicating an easy success, we reduce $\lambda$ by a constant factor $F > 1$ (of course, not letting $\lambda$ drop below $1$). If an iteration was not successful, we increase $\lambda$ by a factor of $F^{1/4}$ (since we analyze the algorithm for mutation probability $p=\lambda/n$ we do not let $\lambda$ exceed $n$). Consequently, after a series of iterations with an average success rate of $1/5$, we end up with the initial value of $\lambda$ (unless the lower barrier of $1$ was hit). 

\begin{algorithm2e}[t]%
	\textbf{Initialization:} 
	Sample $x \in \{0,1\}^n$ uniformly at random and query $f(x)$\;
	Initialize $\lambda \assign 1$\;
 \textbf{Optimization:}
\For{$t=1,2,3,\ldots$}{
\underline{\textbf{Mutation phase:}}\\
\Indp
Sample $\ell$ from $\Bin(n,p)$\;
\For{$i=1, \ldots, \lambda$}{
Sample $x^{(i)} \assign \mut_{\ell}(x)$ and query $f(x^{(i)})$\;
}
Choose $x' \in \{x^{(1)}, \ldots, x^{(\lambda)}\}$ with $f(x')=\max\{f(x^{(1)}), \ldots, f(x^{(\lambda)})\}$ u.a.r.\;
\Indm
\underline{\textbf{Crossover phase:}}\\
\Indp
\For{$i=1, \ldots, \lambda$}{
Sample $y^{(i)} \assign \cross_{c}(x,x')$ and query $f(y^{(i)})$\;
}
If exists, choose $y \in \{y^{(1)}, \ldots, y^{(\lambda)}\} \setminus \{x\}$ 
with 
$f(y)=\max\{f(y^{(1)}), \ldots, f(y^{(\lambda)})\}$ u.a.r.; otherwise, set $y:=x$\;
\Indm
\underline{\textbf{Selection and update step:}}\\
\Indp
\lIf{$f(y)>f(x)$}{
$x \assign y$; $\lambda \assign \max\{\lambda/F,1\}$\;}
\lIf{$f(y)=f(x)$}{
$x \assign y$; $\lambda \assign \min\{\lambda F^{1/4},n\}$\;}
\lIf{$f(y)<f(x)$}{$\lambda \assign \min\{\lambda F^{1/4},n\}$\;}
\Indm
}
\caption{The self-adjusting \ga with mutation probability $p$, crossover probability $c$, and update strength $F$.}
\label{alg:GAself}
\end{algorithm2e}

As a technical remark we note that where an integer is required in Algorithm~\ref{alg:GAself} (e.g., lines~6 and~10) we round $\lambda$ to its closest integer, i.e., instead of $\lambda$ we regard $\lfloor \lambda \rfloor= \lambda - \{\lambda\}$ if the fractional part $\{\lambda\}$ of $\lambda$ is less then $1/2$ and we regard $\lceil \lambda \rceil:=\lfloor \lambda\rfloor +1$ otherwise.

\subsection{Runtime Analysis}
\label{sec:runtime}

We show that the self-adjusting \ga (with standard parameters $p = \lambda/n$ and $c = 1/\lambda$) solves the generalized \onemax problem in linear time when the self-adjusting speed factor $F$ is not too large. 

The proof of this result is rather technical. For this reason, we are only able to show a linear optimization time when $F$ is smaller than a certain constant $F^*$, but we do not make this $F^*$ precise. In general, making implicit constants precise is a difficult task in runtime analysis, and for many much simpler problems the implicit constants are not known. In the experiments conducted in~\cite{DoerrDE13}, see in particular Figure~4 there, all values $F \in [1,2]$ worked well (recall that in Auger's implementation~\cite{Auger09} $F = 1.5$ was used). At the end of this section, we give some indication why $F$-values larger than $2.25$, however, may lead to an exponential expected optimization time.

Our main result is the following.
\begin{theorem}
\label{thm:selfadaptive}
The optimization time of the self-adjusting \ga with parameters $p=\lambda/n$ and $c=1/\lambda$ on every generalized \onemax function is $O(n)$ for any sufficiently small update strength $F>1$.
\end{theorem}
 
To prove Theorem~\ref{thm:selfadaptive}, roughly speaking, we show that the population sizes $\lambda$ suggested by the one-fifth success rule are usually not very far from the fitness-dependent choice $\lambda^*=\lceil\sqrt{n/(n-f(x))}\rceil$ analyzed in~\cite[Theorem~8]{DoerrDE15} (which is restated above as Theorem~\ref{thm:fitnessdependent}). 
Intuitively, if $\lambda$ happens to be much larger than $\lambda^*$, the success probability of the \ga is so large that with reasonably large probability one of the next iterations is successful and, as a consequence, the value of $\lambda$ is then adjusted to its previous value divided by $F$, thus approaching again $\lambda^*$.
A key argument in this proof is the following lemma, which shows that for large values of $\lambda$ the success probability of the \ga is indeed reasonably large. This lemma can be seen as a generalization of Lemma~7 in~\cite{DoerrDE15}. 

\begin{lemma}
\label{lem:boundp}\label{LEM:BOUNDP}
Let $x \in \{0,1\}^n$.
Let $\lambda \geq \const \lceil\sqrt{n/(n-f(x))}\rceil$. 
Let $q=q(\lambda)$ be the probability that one iteration of Algorithm~\ref{alg:GA} (with parameters $p=\lambda/n$ and $c=1/\lambda$) starting in $x$ is successful. 
There exists a constant $C$ such that for all $\const>C$ we have $q>1/5$.
\end{lemma}

\begin{proof}[of Lemma~\ref{lem:boundp}]
We use the same notation as in the description of Algorithm~\ref{alg:GAself}. For readability purposes we again write $\lambda$ even if an integer is required. 
For any fixed $\varepsilon>0$ and for any $\lambda$, the success probability $q$ of increasing the fitness by at least one is (by the law of total probability) at least
\begin{align}
\label{eq:lowerboundp}
& \Pr[L \in [(1-\varepsilon)\lambda, (1+\varepsilon)\lambda]]\cdot\\
& \min\{\Pr[f(y)>f(x) \mid L=\ell] \mid \ell \in [(1-\varepsilon)\lambda, (1+\varepsilon)\lambda]\}.\nonumber
\end{align}
By Lemmas~\ref{lem:phase1} and~\ref{lem:phase2} it holds for any $\ell$ that
\begin{align*}
& \Pr[f(y)>f(x) \mid L=\ell] 
\geq \\
& \nonumber 
\left( 1- \left(\tfrac{f(x)}{n}\right)^{\lambda \ell}\right)
\left(1-\left(1-\tfrac{1}{\lambda} (1-\tfrac{1}{\lambda})^{\ell}\right)^{\lambda}\right).
\end{align*}

It thus suffices to bound from below
\begin{itemize}
	\item[(i)]  $\min\{\left(1-\left(1-\tfrac{1}{\lambda} (1-\tfrac{1}{\lambda})^{\ell}\right)^{\lambda}\right) \mid \ell \in [(1-\varepsilon)\lambda, (1+\varepsilon)\lambda] \}$,
	\item[(ii)]  $\min\{1- \left(\tfrac{f(x)}{n}\right)^{\lambda \ell} \mid \ell \in [(1-\varepsilon)\lambda, (1+\varepsilon)\lambda]\}$,
	\item[(iii)]  $\Pr[L \in [(1-\varepsilon)\lambda, (1+\varepsilon)\lambda]]$.
\end{itemize}

\textbf{Bounding (i): } 
For any $\ell \in [(1-\varepsilon)\lambda, (1+\varepsilon)\lambda]$ it holds that
\begin{align*}
\left(1-\tfrac{1}{\lambda} (1-\tfrac{1}{\lambda})^{\ell}\right)^{\lambda}
& \leq
\left(1-\tfrac{1}{\lambda} (1-\tfrac{1}{\lambda})^{(1+\varepsilon)\lambda}\right)^{\lambda}\\
& \leq 
\left(1-\tfrac{1}{\lambda} (1/4)^{1+\varepsilon}\right)^{\lambda}\\
& \leq 
\exp(-(1/4)^{1+\varepsilon})\,.
\end{align*}
For $\varepsilon < 1/25$ we can thus bound expression (i) from below by 0.21.

\textbf{Bounding (ii): } 
%
We set $d:=n-f(x)$ and obtain, for $\ell \in [(1-\varepsilon)\lambda, (1+\varepsilon)\lambda]$,
\begin{align*}
\left(\tfrac{f(x)}{n}\right)^{\lambda \ell}
& \leq  
\left(\tfrac{f(x)}{n}\right)^{(1-\varepsilon)\lambda^2}
\leq 
\left(1-\tfrac{d}{n}\right)^{(1-\varepsilon)\const^2 n/d}\\
& \leq
(1/e)^{(1-\varepsilon) \const^2}\,.
\end{align*}
This expression is at most $1/100$ for large enough $\const$, showing that we can bound (ii) from below by $0.99$.

\textbf{Bounding (iii): } We apply Chernoff's bound to bound (iii) from below by 
$1-2\exp(-\varepsilon^2\lambda/3)$. Since $\lambda \geq 2 \const$, this term is again larger than $0.99$ for a suitably chosen $\const$. 

Putting everything together we have seen that, for a suitable choice of $\const$, the expression in (\ref{eq:lowerboundp}) is strictly larger than $0.99^2 \cdot 0.21 >1/5$.
\end{proof}

While the proof of Lemma~\ref{lem:boundp} was rather straightforward, the proof of the main theorem, i.e., Theorem~\ref{thm:selfadaptive}, is much more involved. 

\begin{proof}[of Theorem~\ref{thm:selfadaptive}]
As in the overview given before Lemma~\ref{lem:boundp} we sloppily denote in the following by $\lambda^*$ our fitness-dependent parameter choice of Theorem~\ref{thm:fitnessdependent}; i.e., $\lambda^*:=\lceil\sqrt{n/(n-f(x))}\rceil$. Note that the value of $\lambda^*$ depends on the current fitness value but that this is not reflected in the abbreviation. To increase the readability, we omit again to specify whether $\lambda$ has to be rounded up or down.

We partition the optimization process into phases. 
The first phase starts with the first fitness evaluation. 
A phase ends with an iteration at whose end we have increased the fitness of $x$ 
and $\lambda \leq \const \lambda^*$ holds, for a sufficiently large constant $\const$ that we do not compute explicitly. ($\const$ is determined by Lemma~\ref{lem:boundp}.) 

We shall first show that each phase has an expected cost of $O(\lambda^*)$ fitness evaluations. From this is it not difficult to conclude the proof by arguments used in the proof of Theorem~\ref{thm:fitnessdependent}. 

To bound the expected cost of each phase, we distinguish between \emph{``short''} phases, in which $\lambda < \const \lambda^*$ holds throughout, and \emph{``long''} phases, in which $\lambda \geq \const \lambda^*$ for at least one iteration. We abbreviate the threshold $\const \lambda^*$ by $\barlam$. Note that $\barlam$ as well depends on the current fitness $f(x)$.

\textbf{Claim 1: } The expected cost of a short phase is $O(\lambda^*)$.

\emph{Proof of claim 1: } 
Let $\tilde{\lambda}$ be the value of $\lambda$ at the beginning of the phase and 
let $t$ be the number of iterations of the phase. 
Since $\lambda$ does not exceed the threshold $\barlam$, there is exactly one iteration in which the fitness of $x$ is increased. That is, the value of $\lambda$ has first been multiplied $t-1$ times by $F^{1/4}$ until there was a fitness increase and the $\lambda$-value has been shrunk as a consequence of the fitness increase. 
The value of $\lambda$ at the end of the phase is thus $\tilde{\lambda} F^{(t-1)/4 - 1}$ and this value is bounded from above by $\barlam$ (since we are considering a short phase). 
Since an iteration with parameter $\lambda$ requires $2\lambda$ fitness evaluations, the total number of fitness evaluations is thus
\begin{align}
\label{eq:nochmal}
\sum_{i=0}^{t-1}{2\tilde{\lambda} F^{i/4}} 
= 2 \tilde{\lambda} \frac{F^{t/4}-1}{F^{1/4}-1}
= O(\barlam)
= O(\lambda^*).
\end{align}

\textbf{Claim 2: } The expected cost of a long phase is $O(\lambda^*)$.

\emph{Proof of claim 2: } 
We split the long phase into an opening phase and a main phase. The \emph{opening phase} ends with the last iteration in which $\lambda < \barlam$ holds, so that the \emph{main phase} starts with a $\lambda$ that is at least as large as $\barlam$, but less than $\barlam F^{1/4}$. 

Let $T$ denote the number of fitness evaluations during the phase and let $I$ denote the number of iterations in the main~(!) phase.
As in the proof of Claim 1 it will be easy to see that 
$\E[T \mid I=t] = D \lambda^* F^{t/4}$ for all $t$ and some fix constant $D$.
The most technical part of this proof is to bound the probability that the main phase of a long phase requires $t$ iterations; i.e., $\Pr[I=t]$ given that we are in a long phase.
We show that this probability is at most $\exp(-ct)$ for some positive constant $c$. It is well known that the geometric series $(\sum_{\ell=1}^{t}{F^{t/4}\exp(-ct)})_{\ell \in \Z_{>0}}$ converges if $F^{1/4} < \exp(c)$. 
The overall expected number of fitness evaluations during a long phase is thus
\begin{align}
\nonumber
\sum_{t=1}^{\infty}{\E[T \mid I=t]\Pr[I=t]}
& \leq 
D \lambda^*  \sum_{t=1}^{\infty}{F^{t/4}\exp(-ct)},
\end{align}
which is $O(\lambda^*)$ as desired.

It remains to prove the following two claims.

\textbf{Claim 2.1: } $\E[T \mid I=t] \leq D \lambda^* F^{t/4}$ for some large enough constant $D$.

\textbf{Claim 2.2: } Given that we are in a long phase, we have $\Pr[I=t] = \exp(-ct)$ for a positive constant~$c$.

\emph{Proof of Claim 2.1: }
The cost of the opening phase is at most 
$2 \tilde{\lambda} \sum_{i=0}^{m}{F^{i/4}}$, where $\tilde{\lambda}$ is the initial value of $\lambda$ at the beginning of the phase and $m$ is chosen maximally such that $\lambda F^{m/4} < \barlam$. 
As in (\ref{eq:nochmal}) one shows that this sum is $O(\lambda^*)$.
Similarly, since the initial $\lambda$ of the main phase is at most $\barlam F^{1/4}$, the cost of the main phase is at most
\begin{align*}
\barlam \sum_{i=1}^{t}{F^{i/4}} 
& = \barlam (F^{(t+1)/4}-1)/(F^{1/4}-1) \\
& \leq D' \barlam F^{(t+1)/4}
\end{align*}
for $D' \geq 1/(F^{1/4}-1)$. 

\emph{Proof of Claim 2.2: }
As mentioned above, the main phase starts with a $\tilde{\lambda}$ that is at least $\barlam$ and strictly less than $\barlam F^{1/4}$. We are interested in the first point in time at which $\lambda$ is less than $\barlam$. 
Note that all future values encountered in this phase are of the type $\tilde{\lambda}F^r$  for $r$ being a multiple of $1/4$. By regarding this exponent $r$, we transform the process into a biased random walk on the line $(1/4)\Z$. Our starting position is $0$. If an iteration is successful, i.e., if the fitness value of $x$ has increased during the iteration, the process does one step of length one to the left. It does a step of length $1/4$ to the right otherwise (we thus pessimistically ignore the fact that $\lambda$ never exceeds $n$). We bound the probability that it takes $t$ or more iterations until this random walk has reached a value of less then $0$. 
At this point in time the current $\lambda$ is less than the original $\barlam$ value that was active at the beginning of the main phase. That is, when the random walk reaches a position smaller than $0$, $\lambda$ is for certain less then the then active threshold $\barlam$ (which, by definition, increases whenever the fitness value of $x$ does). 

If an iteration is successful with probability at least $q$, the expected progress of this random walk in one iteration is $(1-q)/4 - q$, which is negative since by Lemma~\ref{lem:boundp} we have $q > 1/5$. 
Hence there exists a constant $c>0$ such that $(1-q)/4 - q < -c$.

To conclude the proof of Claim 2.2, let us define random variables $X_i$, $1 \leq i \leq t$, by setting $X_i=1/4$ if the fitness does not increase in the $i$th iteration of the main phase, and setting $X_i=-1$ otherwise. 
We have just seen that $\E[X_i]\leq -c$. 
Given that we are in a long phase, the probability that the main phase has length at least $t$ equals 
$\Pr[\forall j < t: \sum_{i=1}^{j}{X_i}>0]$. This is at most $\Pr[\sum_{i=1}^{t-1}{X_i}>0]$, which is in turn bounded from above by
$\Pr[\sum_{i=1}^{t-1}{X_i}> \E[\sum_{i=1}^{t-1}{X_i}]+(t-1)c]$. We apply Chernoff's bound---confer Theorem 1.11 in~\cite{Doerr11bookchapter} for a version that allows for random variables that do not necessarily take positive values---to see that, as desired, this term is at most
\begin{align*}
\exp\left(-\frac{2(t-1)^2c^2}{(t-1)(5/4)^2}\right)
= 
\exp\left(-32 (t-1) c^2/25\right).
\end{align*}
\end{proof}

\subsection{General Insights from the Runtime Analysis} 
\label{sec:insights}
The analysis above reveals the following facts, which might be helpful in general when trying to use a one-fifth success rule or a related self-adjusting rule in discrete search spaces. 

\textbf{\emph{The adjustment rule must fit to the limiting success probability.}} In the proof above, it was crucial that the success probability shown in Lemma~\ref{lem:boundp} was a constant larger than $1/5$. It is easy to see that if the success probability was uniformly bounded from above by a constant $\sigma < 1/5$, then $\log_F(\lambda)$ in expectation would increase by a positive constant in each round. Consequently, $\lambda$ would show an exponential growth, quickly leading to wastefully large values. This can partially be overcome by imposing an upper barrier for $\lambda$ (we have such a barrier, namely $\lambda \le n$, to ensure that the mutation probability $p = \lambda/n$ is at most $1$), however, this would still lead to the algorithm mostly working with this maximal value of $\lambda$ instead of a value close to the ideal $\lambda^*$.

In general, there is no reason for not trying success rules with other ratios $1/r$ than one-fifth, that is, increasing $\lambda$ by $F^{1/(r-1)}$ instead of $F^{1/4}$ in case of non-success. In general, a larger value of $r$ will slightly decrease the speed of adjustment, but is more likely to overcome the problem described in the previous paragraph. Note that for our problem, when $\lambda = \omega(1)$, the success probability is uniformly bounded from above by $(1+o(1)) e^{-1/e} \approx 0.31$. Consequently, the one-fifth rule avoids the exponential growth of $\lambda$, whereas a one-third rule or a one-half rule (e.g., doubling or halving the parameter as in~\cite{LassigS11}) would not. 

\textbf{\emph{The constant $F$ matters.}} Even when the combination of update rule and success probability avoids an expected exponential growth of $\lambda$, things can still go wrong when the update strength $F$ is too large. Here is an example (where, to ease the presentation, we assume that we have no upper barrier on $\lambda$; with an upper barrier, as above, the problem remains, though possibly to a smaller extent): 
Imagine that we start with some value $\lambda_0$. Above, we saw that the success probability of an iteration is bounded from above by $0.31$. Consequently, the probability of having exactly $m$ consecutive non-successes is at least $0.31 \cdot 0.69^m$. The optimization time of the last iteration alone is $\lambda_0 F^{m/4}$. Consequently, the expected effort of finding one improvement is at least $0.31 \lambda_0 \sum_{m = 0}^\infty (0.69 F^{1/4})^m$. When $0.69 F^{1/4} > 1$, this series does not converge, i.e., the expected effort for one improvement is infinite. In our case, this happens (at least) when $F > 4.41$. 
Note that this was a rough estimate aimed at quickly demonstrating that large $F$-values can be dangerous. Better values can be achieved with more effort. E.g., the probability that among $6m$ iterations, we have at most $m$ successes, is more than $(0.755+o(1))^{m}$; this can be seem from approximating the binomial distribution with a normal distribution. 
Since this event also increases the initial $\lambda$ value by $F^{m/4}$, the corresponding series already diverges for $F \ge 3.08$. Optimizing the ratio of successes and non-successes, we see that the probability of having $\gamma$ successes among $1+5\gamma$ trials is more than $(0.8167+o(1))^m$, showing that any $F \ge 2.25$ leads to an exponential expected optimization time. We do not know to what extent this argument can be improved. For this reason, we would rather suggest to choose a small value of $F$, clearly below $2$, and trade in the possibly faster adjustment to the ideal parameter value for a reduced risk of an expected infinite optimization time. 

\section{A Linear Lower Bound for All Possible Parameter Choices}
\label{sec:lower}

In the previous section we have seen that the self-adjusting \ga is faster on average than with any static population size. We next show that it is asymptotically best possible also among all dynamic parameter choices. That is, regardless of how the parameters are updated in each iteration, the \ga always has an expected runtime on \onemax that is at least linear in $n$. 

\begin{theorem}
\label{thm:lower}
For every (possibly dynamic) choice of the mutation probability $p$, the crossover probability $c$, and the population size $\lambda$, the \ga performs at least linearly many function evaluations on average before it evaluates for the first time the unique global optimum. That is, regardless of the parameter update scheme, the expected runtime of the \ga on \onemax is $\Omega(n)$.
\end{theorem}

For parameter choices that do not depend on absolute fitness values, Theorem~\ref{thm:lower} follows from an elegant technique in black-box complexity. Since we believe this intuitive argument to be of general interest to the Self-* research community, we present it in the following section. It can be seen as a nice, yet powerful tool for analyzing the limitations of evolutionary and other black-box optimization algorithms.

For fitness-dependent parameter choices, Theorem~\ref{thm:lower} also holds, but needs a different argument as we shall comment in Section~\ref{sec:lowerfitness}.

\subsection{Lower Bound for Self-Adjusting Parameter Choices}
\label{sec:lowerblackbox}


We start the exposition of the lower bound by introducing the concept of \emph{comparison-based algorithms}.

\begin{definition}
\label{def:comparison}
A comparison-based black-box algorithm does not make use of absolute fitness values. Instead, it bases all decisions solely on the comparison of search points. 
\end{definition} 

To see that the self-adjusting \ga is a comparison-based algorithm, we reformulate the algorithm slightly, see Algorithm~\ref{alg:ga2}. This alternative presentation shows that indeed all decisions of the \ga are entirely based on comparisons between at most two search points. A lower bound that holds for \emph{all} comparison-based algorithms thus immediately implies a lower bound for the \ga. We therefore expand our view and regard the whole class of comparison-based algorithms. Theorem~\ref{thm:lower} for non-fitness-dependent parameter choices follows from the following statement, which is folklore knowledge in black-box complexity and has been formally stated (in much more general form) in~\cite[Corollary 2]{TeytaudG06}.

\begin{algorithm2e}%
	\textbf{Initialization:} 
	Sample $x \in \{0,1\}^n$ uniformly at random\;
 \textbf{Optimization:}
\For{$t=1,2,3,\ldots$}{
Depending on $x$ (but not $f(x)$) choose $\lambda \in [n]$, $p \in [0,1]$, and $c\in [0,1]$\;
\underline{\textbf{Mutation phase:}}\\
\Indp
Sample $\ell$ from $\Bin(n,p)$ and $x':=x^{(1)} \assign \mut_{\ell}(x)$\;
\For{$i=2, \ldots, \lambda$}{
Sample $x^{(i)} \assign \mut_{\ell}(x)$\;
\lIf{$f(x^{(i)}) \geq f(x')$}{$x' \assign x^{(i)}$\;}
}
\Indm
\underline{\textbf{Crossover phase:}}\\
\Indp
Initialize $y:=y^{(1)} \assign \cross_{c}(x,x')$\;
\For{$i=2, \ldots, \lambda$}{
Sample $y^{(i)} \assign \cross_{c}(x,x')$\;
\lIf{\normalfont{(}$y^{(i)} \neq x$ \normalfont{and} $f(y^{(i)}) \geq f(y)$\normalfont{)}}{$y \assign y^{(i)}$\;}}
\Indm
\underline{\textbf{Selection and update step:}}\\
\Indp
\lIf{$f(y) \geq f(x)$}{$x \assign y$\;}
\Indm
}
\caption{A reformulation of the \ga with possibly adaptive (but not fitness-dependent) parameters $p$, $c$, and $\lambda$.}
\label{alg:ga2}
\end{algorithm2e}

\begin{theorem}
\label{thm:lowerboundcompbased}
Every comparison-based algorithm needs at least $\Omega(n)$ comparisons on average to optimize a generalized \onemax function.
\end{theorem}

The intuitive argument for Theorem~\ref{thm:lowerboundcompbased} is pretty simple. In order to optimize a \onemax function $\OM_z$ we need to identify $z$. That is, we need to \emph{learn} the $n$ bits of $z$. Roughly speaking, with each query we learn at most one bit of information about $z$, namely in the mutation phase we learn whether $\OM_z(x^{(i)}) \geq \OM_z(x')$ or not, and in the crossover phase we learn the bit whether or not $\OM_z(y^{(i)}) \geq \OM_z(y)$. Finally, we learn the one bit of information whether or not $\OM_z(y) \geq \OM_z(x)$. Thus one iteration with $\lambda$ offspring in the mutation and the crossover phase each gives us a total of $2(\lambda-1)+1$ bits of information. Thus, amortized over the $2\lambda$ offspring that were created, this is a bit less than one bit of information per query. Since we need to learn all $n$ bits of $z$, this shows (intuitively) that we need to sample and compare at least $n$ search points in total. This implies the lower bound. 

It is not too difficult to make this intuitive argument formal. To this end, one employs Yao's Minimax Principle~\cite{Yao77}, a powerful tool in black-box complexity. The interested reader can find a quite accessible exposition of Yao's Principle along with some easy to follow examples from evolutionary computation in~\cite{DrosteJW06}.

\subsection{Lower Bound for Fitness-Dependent Parameter Choices}
\label{sec:lowerfitness}

The proof given in Section~\ref{sec:lowerblackbox} does not work for parameter choices that possibly depend on the absolute fitness of intermediate search points. Intuitively, the problem here is that $\OM_z(x) \in [0..n]$ and knowing (or, rather, using knowledge about) the absolute fitness value of $x$ provides $\log_2(n+1)$ bits of information. This would only yield a lower bound of order $n/\log n$. Still the statement, i.e., Theorem~\ref{thm:lower}, holds also for fitness-dependent parameter choices, as we shall briefly comment in this section. Note that this bound also implies the optimality of the fitness-dependent choice suggested in~\cite{DoerrDE15}.

In the analysis of~\cite{DoerrD15exact} it is shown (for the recommended parameter choices $p=\lambda/n$ and $c=1/\lambda$) that the expected fitness gain in one iteration of the \ga with population size $\lambda$ is of order at most $\log \lambda/\log \log \lambda$ (see the proof of the upper and lower bounds of Theorem~2 in~\cite{DoerrD15exact}). That is, for ``investing'' $2 \lambda$ function evaluations we obtain a fitness gain of order at most $\log \lambda/\log \log \lambda$. 
This shows that the average fitness gain per function evaluation cannot be more than constant, thus implying the linear lower bound. To make this formal, one can use the additive drift theorem~\cite{HeY01}.

To show Theorem~\ref{thm:lower} in full generality, one has to redo the analysis from~\cite{DoerrD15exact} for general $p$ and $c$, a tedious but straightforward work that we omit here. 

\section{Conclusions}
\label{sec:conclusions}
We have analyzed the \ga with self-adjusting population sizes. We have shown that it optimizes any generalized \onemax function in linear time. This is best possible for any (static or dynamic) parameter choice and is better by a $\Theta(\sqrt{\log(n)\log\log\log(n)/\log\log(n)})$ factor than any \ga with static population size. Our result thus shows for the first time that self-adjusting parameter choices can be provably beneficial in discrete optimization problems. 

We hope that our work inspires more work on the running time of evolutionary algorithms with self-adjusting parameter choices. We have provided some general insights that should be regarded when implementing a one-fifth success rule in discrete search spaces. 

While our work focuses on an adjustment of the population size, we are confident that also for other parameters of evolutionary algorithms, e.g., the mutation or crossover rates, self-adjusting choices can be analyzed theoretically. 

\section*{Acknowledgments}
The authors would like to thank Anne Auger and Nikolaus Hansen for valuable discussions on various aspect of this project.

Figure~\ref{fig:plot} is taken from~\cite{DoerrDE15}. 
}


\begin{thebibliography}{10}

\bibitem{Auger09}
A.~Auger.
\newblock Benchmarking the (1+1) evolution strategy with one-fifth success rule
  on the {BBOB}-2009 function testbed.
\newblock In {\em Proc. of GECCO'09 (Companion)}, pages 2447--2452. ACM, 2009.

\bibitem{AugerH13}
A.~Auger and N.~Hansen.
\newblock Linear convergence on positively homogeneous functions of a
  comparison based step-size adaptive randomized search: the (1+1) {E}{S} with
  generalized one-fifth success rule.
\newblock {\em CoRR}, abs/1310.8397, 2013.
\newblock Available online at \url{http://arxiv.org/abs/1310.8397}.

\bibitem{BottcherDN10}
S.~B{\"o}ttcher, B.~Doerr, and F.~Neumann.
\newblock Optimal fixed and adaptive mutation rates for the leadingones
  problem.
\newblock In {\em Proc. of PPSN'10}, volume 6238 of {\em LNCS}, pages 1--10.
  Springer, 2010.

\bibitem{DeJong75}
K.~A. De~Jong.
\newblock {\em An Analysis of the Behavior of a Class of Genetic Adaptive
  Systems}.
\newblock PhD thesis, 1975.

\bibitem{Devroye72}
L.~Devroye.
\newblock {\em The compound random search}.
\newblock PhD thesis, 1972.

\bibitem{Doerr11bookchapter}
B.~Doerr.
\newblock Analyzing randomized search heuristics: Tools from probability
  theory.
\newblock In A.~Auger and B.~Doerr, editors, {\em Theory of Randomized Search
  Heuristics}, pages 1--20. World Scientific Publishing, 2011.

\bibitem{DoerrD15exact}
B.~Doerr and C.~Doerr.
\newblock A tight bound for the runtime of the $(1+(\lambda, \lambda))$ genetic
  algorithm on onemax.
\newblock In {\em Proc. of GECCO'15}. ACM, 2015.
\newblock To appear.

\bibitem{DoerrDE13}
B.~Doerr, C.~Doerr, and F.~Ebel.
\newblock Lessons from the black-box: {F}ast crossover-based genetic
  algorithms.
\newblock In {\em Proc. of GECCO'13}, pages 781--788. ACM, 2013.

\bibitem{DoerrDE15}
B.~Doerr, C.~Doerr, and F.~Ebel.
\newblock From black-box complexity to designing new genetic algorithms.
\newblock {\em Theoretical Computer Science}, 567:87--104, 2015.

\bibitem{DoerrJSWZ13}
B.~Doerr, T.~Jansen, D.~Sudholt, C.~Winzen, and C.~Zarges.
\newblock Mutation rate matters even when optimizing monotonic functions.
\newblock {\em Evolutionary Computation}, 21:1--27, 2013.

\bibitem{DrosteJW02}
S.~Droste, T.~Jansen, and I.~Wegener.
\newblock On the analysis of the (1+1) evolutionary algorithm.
\newblock {\em Theoretical Computer Science}, 276:51--81, 2002.

\bibitem{DrosteJW06}
S.~Droste, T.~Jansen, and I.~Wegener.
\newblock Upper and lower bounds for randomized search heuristics in black-box
  optimization.
\newblock {\em Theory of Computing Systems}, 39:525--544, 2006.

\bibitem{EibenHM99}
A.~E. Eiben, R.~Hinterding, and Z.~Michalewicz.
\newblock Parameter control in evolutionary algorithms.
\newblock {\em IEEE Transactions on Evolutionary Computation}, 3:124--141,
  1999.

\bibitem{EibenS03}
A.~E. Eiben and J.~E. Smith.
\newblock {\em Introduction to Evolutionary Computing}.
\newblock Springer Verlag, 2003.

\bibitem{HansenGO95}
N.~Hansen, A.~Gawelczyk, and A.~Ostermeier.
\newblock Sizing the population with respect to the local progress in (1,$
  \lambda$)-evolution strategies - a theoretical analysis.
\newblock In {\em Proc. of CEC'95}, pages 80--85. IEEE, 1995.

\bibitem{HeY01}
J.~He and X.~Yao.
\newblock Drift analysis and average time complexity of evolutionary
  algorithms.
\newblock {\em Artificial Intelligence}, 127:57--85, 2001.

\bibitem{Jagerskupper05}
J.~J{\"a}gersk{\"u}pper.
\newblock Rigorous runtime analysis of the (1+1) {ES}: 1/5-rule and ellipsoidal
  fitness landscapes.
\newblock In {\em Proc. of FOGA'05}, volume 3469 of {\em LNCS}, pages 260--281.
  Springer, 2005.

\bibitem{KarafotiasHE15}
G.~Karafotias, M.~Hoogendoorn, and A.~Eiben.
\newblock Parameter control in evolutionary algorithms: Trends and challenges.
\newblock {\em IEEE Transactions on Evolutionary Computation}, 19:167--187,
  2015.

\bibitem{LassigS11}
J.~L{\"a}ssig and D.~Sudholt.
\newblock Adaptive population models for offspring populations and parallel
  evolutionary algorithms.
\newblock In {\em Proc. of FOGA'11}, pages 181--192. ACM, 2011.

\bibitem{MotwaniR95}
R.~Motwani and P.~Raghavan.
\newblock {\em Randomized Algorithms}.
\newblock Cambridge University Press, 1995.

\bibitem{Rechenberg}
I.~Rechenberg.
\newblock {\em Evolutionsstrategie}.
\newblock Friedrich Fromman Verlag (G{\"u}nther Holzboog {KG}), 1973.

\bibitem{SchumerS68}
M.~Schumer and K.~Steiglitz.
\newblock Adaptive step size random search.
\newblock {\em IEEE Transactions on Automatic Control}, 13(3):270--276, 1968.

\bibitem{Sudholt12}
D.~Sudholt.
\newblock Crossover speeds up building-block assembly.
\newblock In {\em Proc. of GECCO'12}, pages 689--702. ACM, 2012.

\bibitem{TeytaudG06}
O.~Teytaud and S.~Gelly.
\newblock General lower bounds for evolutionary algorithms.
\newblock In {\em Proc. of PPSN'06}, volume 4193 of {\em LNCS}, pages 21--31.
  Springer, 2006.

\bibitem{Yao77}
A.~C.-C. Yao.
\newblock Probabilistic computations: Toward a unified measure of complexity.
\newblock In {\em Proc. of FOCS'77}, pages 222--227. IEEE, 1977.

\end{thebibliography}

\end{document}